\documentclass[conference]{IEEEtran}
\IEEEoverridecommandlockouts
\usepackage{cite}
\usepackage{amsmath,amssymb,amsfonts}
\usepackage{graphicx}
\usepackage{textcomp}
\usepackage{xcolor}

\usepackage{subfigure} 
\usepackage{amsthm} 
\newtheorem{theorem}{Theorem} 
\usepackage{algorithm} 
\usepackage{algpseudocode} 
\usepackage{booktabs} 
\usepackage{makecell,multirow,diagbox}
\usepackage[usestackEOL]{stackengine}

\usepackage{authblk}

\def\BibTeX{{\rm B\kern-.05em{\sc i\kern-.025em b}\kern-.08em
    T\kern-.1667em\lower.7ex\hbox{E}\kern-.125emX}}
    
\usepackage{setspace}
\begin{document}

\title{DP-Net: Dynamic Programming Guided Deep Neural Network Compression
}

\author[*]{Dingcheng Yang}
\author[*]{Wenjian Yu}
\author[**]{Ao Zhou}
\author[*]{Haoyuan Mu}
\author[***]{Gary Yao}
\author[**]{Xiaoyi Wang}
\affil[*]{BNRist, Dept. Computer Science \& Tech., Tsinghua University}
\affil[**]{Beijing Engineering Research Center for IoT Software and Systems, Beijing University of Technology}
\affil[***]{Department of Electrical Engineering and Computer Science, Case Western Reserve University}
\affil[ ]{\textit {\{ydc19, muhy17\}@mails.tsinghua.edu.cn;yu-wj@tsinghua.edu.cn}}
\affil[ ]{\textit {S201861539@emails.bjut.edu.cn;wxy@bjut.edu.cn}}
\affil[ ]{\textit {gxy76@case.edu} }
 
\renewcommand\Authands{ and }

\maketitle

\begin{abstract}
In this work, we propose an effective scheme (called DP-Net) for compressing the deep neural networks (DNNs). It includes a novel dynamic programming (DP) based algorithm to obtain the optimal solution of weight quantization and an optimization process to train a clustering-friendly DNN. Experiments showed that the DP-Net allows larger compression than the state-of-the-art counterparts while preserving accuracy. The largest 77X compression ratio on Wide ResNet is achieved by combining DP-Net with other compression techniques. Furthermore, the DP-Net is extended for compressing a robust DNN model with negligible accuracy loss. At last, a custom accelerator is designed on FPGA to speed up the inference computation with DP-Net.
\end{abstract}


\section{Introduction}
Deep neural networks (DNNs) have been demonstrated to be successful on many tasks. However, the size of DNN model has continuously increased while it achieves better performance. As a result, the storage space of DNN becomes a major concern if we deploy it on resource-constrained devices, especially in edge-computing applications like face recognition on mobile phone and autonomous driving.

In recent years, there are many studies on compressing DNN models. The proposed techniques consist of pruning \cite{lecun1990optimal,han2015learning}, knowledge distillation \cite{hinton2015distilling}, low-bit neural network \cite{courbariaux2016binarized}, compact architectures \cite{howard2017mobilenets} and weight quantization \cite{chen2015compressing}.
Two kinds of weight quantization method, vector quantization and scalar quantization, were proposed in \cite{gong2014compressing}. The scalar quantization was then employed in several work on model compression \cite{han2015deep,chen2016compressing,chen2015compressing}.  Low-bit representation of neural network can be regarded as a variant of scalar quantization, which restricts the weights to low-bit floating-point numbers \cite{courbariaux2016binarized}. Therefore, an accurate weight quantization provides the upper bound of performance for the corresponding low-bit representation. 



In previous work, the K-means clustering solved with the Lloyd's algorithm \cite{lloyd1982} was used for the weight quantization. However, the Lloyd's algorithm is heuristic. Its result is not optimal and is sensitive to  the initial solution, as revealed by experiments in 
\cite{han2015deep}. In this work, we consider the K-means clustering in scalar quantization, but do not use the Lloyd's algorithm. 
Notice the K-means clustering of multi-dimensional data is NP-hard \cite{mahajan2009planar}.
A key contribution of this work is a dynamic programming (DP) based algorithm with $O(n^2K)$ complexity, which produces the optimal solution of the scalar quantization problem. Its advantages over the Lloyd's algorithm lead to better performance in DNN  compression.

On the other hand, recent work shows DNNs are vulnerable to adversarial examples \cite{szegedy2013intriguing}, which can be crafted by adding visually imperceptible perturbations on images. A robust model against the adversarial attacks is needed for security-critical tasks, and becomes a focus \cite{zhang2019theoretically}. However, there is few work devoted to compress the robust DNN model. We demonstrate the proposed  DP-Net compression scheme is also applicable to the robust DNN model, enabling large compression with negligible accuracy drop. 

To validate the speedup in inference brought by the DP-Net, we have also designed a custom accelerator with specific implementation of matrix-vector multiplication. The experiments  on FPGA  show that we can accelerate the computation with fully-connected (FC) layer and convolutional layer remarkably.
The major contributions of this work are as follows.
\begin{enumerate}
    \item A DP based algorithm 
    is proposed to obtain the optimal solution of scalar quantization. 
    Based on it and a clustering-friendly training process, a compression scheme called DP-Net is proposed, which exhibits larger compression of DNNs like FreshNet, GoogleNet than Deep K-means \cite{wu2018deep} and other counterparts. Our experiment show that with DP-Net, pruning and Huffman coding, the storage of Wide ResNet is reduced by 77X, while achieving higher accuracy than Deep K-means.
    \item DP-Net is also extended to the state-of-the-art TRADES model for robust DNNs \cite{zhang2019theoretically}, and achieves 16X compression with negligible accuracy loss for ResNet-18.
    \item To show the inference acceleration brought by DP-Net, we have designed a custom accelerator on FPGA. Experiments on GoogleNet reveal that at least 5X speedup can be achieved for inference computations.
    
\end{enumerate}{}
\section{Problem Formulation}

Regard the weights of DNN as a sequence of vectors $W=\left\{W_1, W_2, \cdots, W_m\right\}$,  where $W_i \in \mathbb{R}^{n_i}$. 
 We consider the scalar weight quantization problem and express the clustering result as  $C= [C_1, C_2, \cdots, C_m]  \in \mathbb{R}^{K \times m}$, where $K$ is the number of clusters and column vector $C_i$ contains the $K$ cluster centers. Without the quantization and using 32 bit floating-point numbers, the number of bits used for storage is $32 \sum_{i=1}^m n_i$. After the quantization, each element of $W_i$ is one of the $K$ elements in $C_i$. So, we just need $\log_2 K \cdot \sum_{i=1}^m n_i$ bits to encode the index and $32mK$ bits to store the cluster centers. Fig. \ref{fig:format} provides an example of the storage formats. This scalar quantization leads to the DNN compression ratio:
\begin{equation}
\label{eq:eq1}
    r=\frac{32\sum_{i=1}^m n_i}{\log_2 K\cdot\sum_{i=1}^m n_i +32mK}
\end{equation}{}

\begin{figure}
 \setlength{\abovecaptionskip}{1pt}
 \setlength{\belowcaptionskip}{1pt}
  \centering
    \includegraphics[width=3.3in]{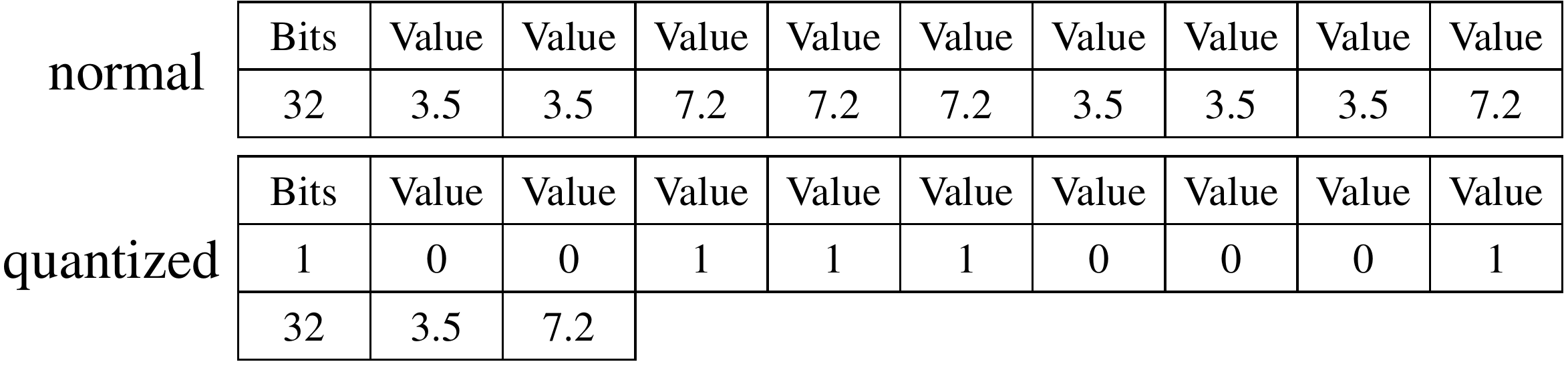}
    \caption{An example of storage formats. The top table represents the normal storage for 9 32-bit floating-point numbers. The bottom table represents the quantized storage format, where only two 32-bit floating-point numbers (3.5 and 7.2) are stored along with 9 bits indicating which value each number is.}
  \label{fig:format}
\end{figure}


Training a DNN with consideration of weight quantization can be formulated as an optimization problem. If the loss function of training DNN is denoted by  $f(W)$, the problem formulation should be:
\begin{eqnarray}
\label{eq:eq2}
 &    \min_{W, C}  f(W),  ~ ~ \\
    & \text{s. t.} ~  ~ C =[C_1, C_2, \cdots, C_m] \in \mathbb{R}^{K \times m}, \nonumber
    \\ & W=\left\{W_1, W_2, \cdots, W_m\right\}, ~ \text{and} ~  \forall i,j, ~  W_{i,j} \in C_i ~ . \nonumber
\end{eqnarray}
The constraint means every element in column vector $W_i$ appears in the vector $C_i$ containing cluster centers.

\section{The DP-Net Scheme for DNN Compression}

Before presenting the DP-Net scheme, we show the drawback of the conventional weight quantization. In Fig. \ref{fig:vis1},  the weights in the first convolution layer of the FreshNet \cite{chen2016compressing} are visualized. It shows that the weights follow the Gaussian distribution, which is consistent with some Bayesian methods assuming that the neural network parameters are Gaussian \cite{ullrich2017soft} and the phenomenon observed by \cite{chen2016compressing} that the weights of learned convolutional filters are typically smooth.
This means that the learned weights may be unsuitable for clustering or quantization. Thus, the network has to be retrained to compensate  the accuracy loss caused by quantization \cite{han2015deep}. However,  for each weight the cluster it belongs to is fixed during the retraining. This is the same as what's done for HashedNet \cite{chen2015compressing}, i.e  randomly grouping each weight and then training, except for a different initial solution. 
Fixing the clusters means a fixed network architecture. 
This is disadvantageous, because 
recent study has shown that the architecture of a network is more important than the initial solution \cite{liu2018rethinking}.

Inspired by Deep K-means \cite{wu2018deep}, we propose to directly optimize the Lagrangian function of (\ref{eq:eq2}) during the training process. The problem becomes an unconstrained optimization: 
\begin{align}
\label{eq:eq3}
    \min_{W, C} & f(W) + \lambda \sum_{i=1}^m\sum_{j=1}^{n_i} \min_{1 \le k \le K} (W_{i,j}-C_{i,k})^2 ,
\end{align}
where $\lambda$ is the Lagrange multiplier. By minimizing this new loss function, we can train a clustering-friendly network (an example is  shown in Fig. \ref{fig:vis2}). 
This is different from Deep K-means, as we are not optimizing the relaxed form of (\ref{eq:eq3}) and thus achieve better accuracy.
\begin{figure}
 \setlength{\abovecaptionskip}{1pt}
  \centering
    \subfigure[]
    {
    \includegraphics[width=1.6in]{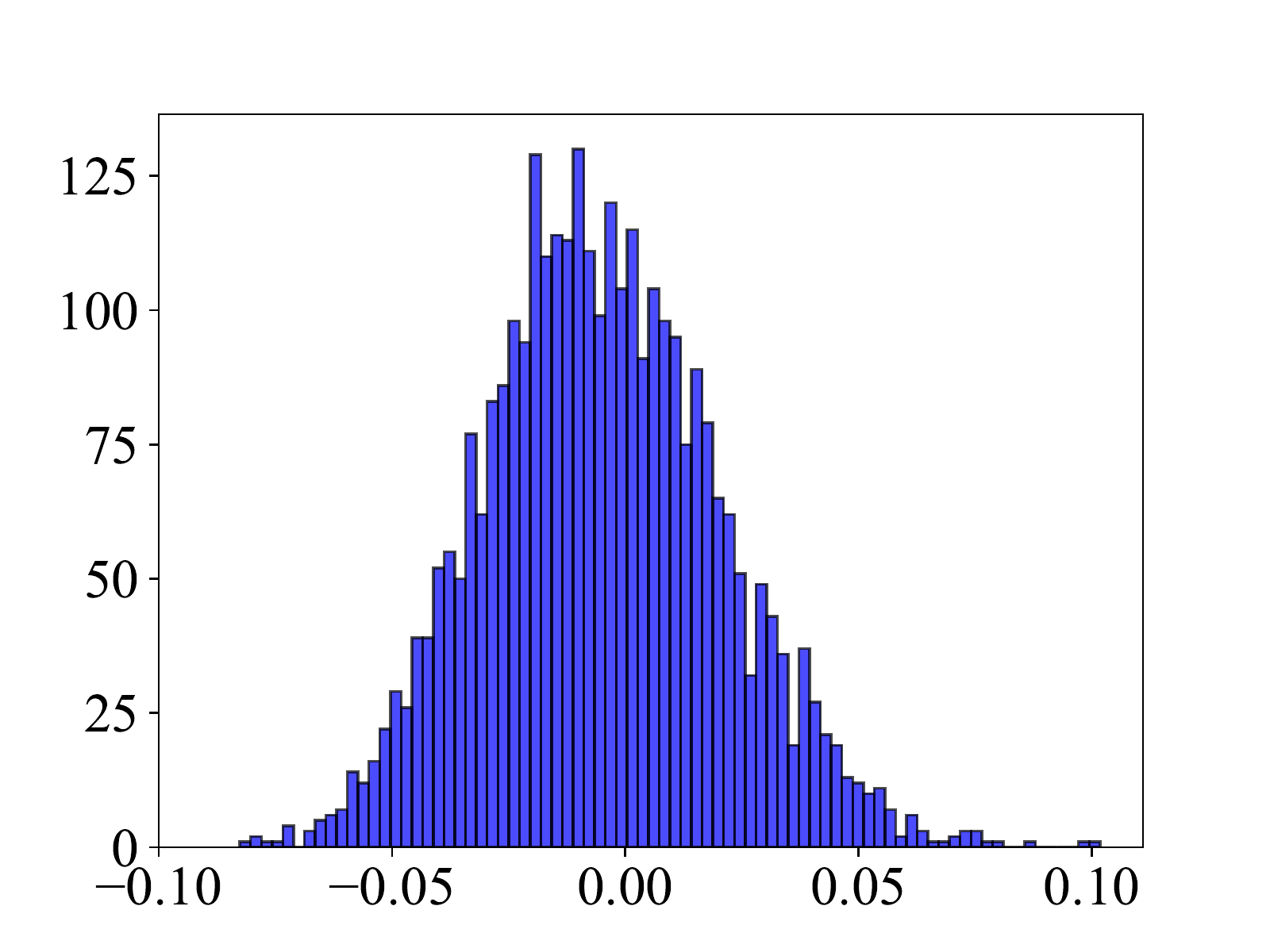}
    \label{fig:vis1}
    }  
    \subfigure[]
    {
    \includegraphics[width=1.6in]{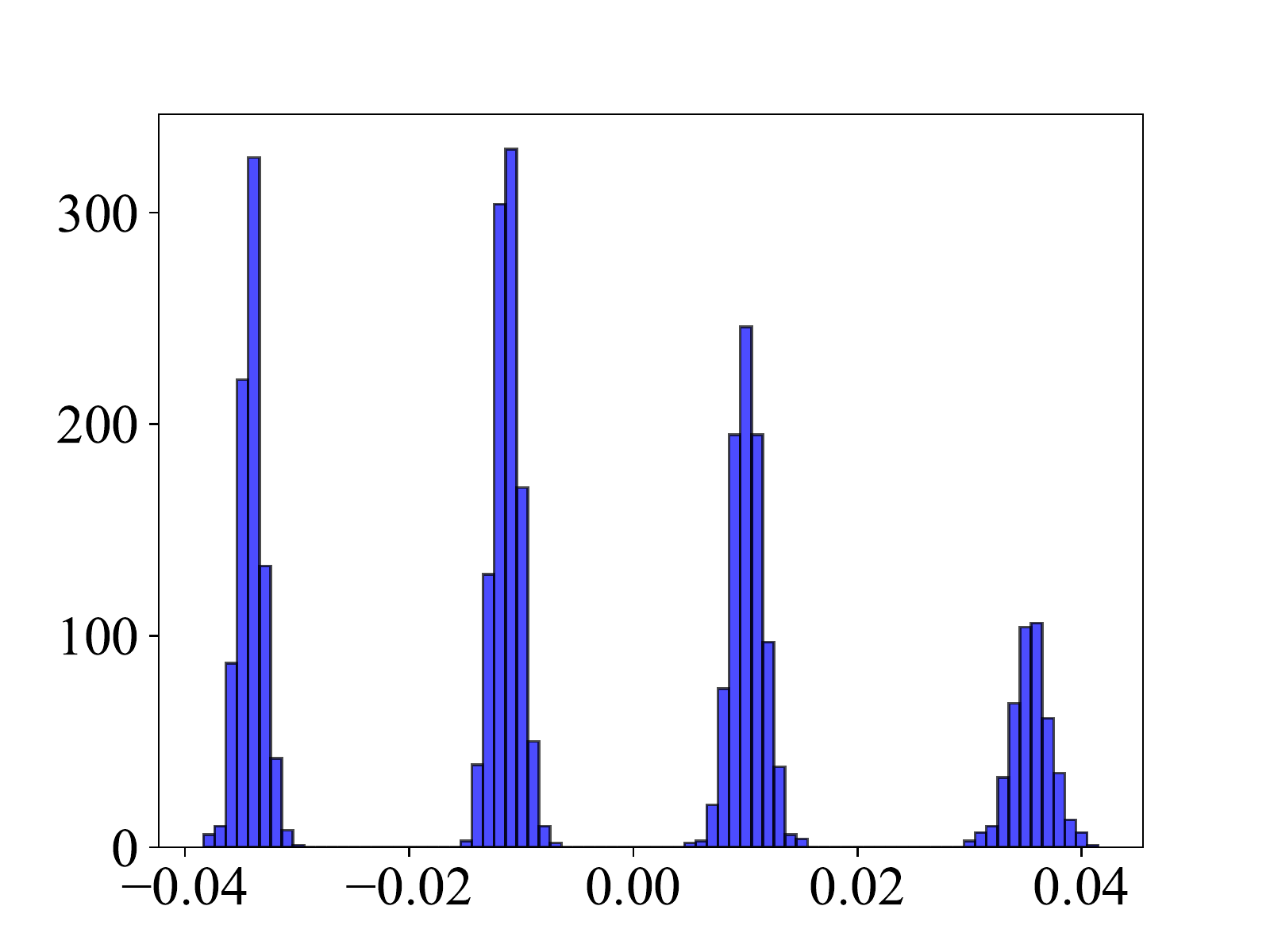}
    \label{fig:vis2}
    }
  \caption{The histograms of  weights in the FreshNet's first convolution layer obtained from (a)  a normal training, and (b) a clustering-friendly training.}
  \label{fig:visualize}
\end{figure}
The idea of DP-Net is to perform such a clustering-friendly training through  alternatively optimizing $W$ and $C$, 
where after every $t$ epochs of stochastic gradient descent (SGD) based optimization of $W$, we  optimize $C$ by solving a scalar K-means clustering. 
The novelty of DP-Net also includes a dynamic programming (DP) based algorithm to obtain the optimal solution of the scalar clustering, which is  employed again to perform the weight quantization after training.

\subsection{DP Based Algorithm for Scalar Quantization}


Although the K-means clustering of $d$-dimensional data is NP-hard for any $d \ge 2$, the optimal solution of scalar quantization (corresponding to the clustering with $d=1$) can be obtained in polynomial time, based on DP and Theorem \ref{thm:theorem1}.

\begin{theorem}{}
\label{thm:theorem1}
Let $x_1 \le x_2 \le \cdots \le x_n$ be $n$ scalars which need to be clustered into $K$ classes. The clustering result is expressed as the index set $p=\{p_1, p_2, \cdots, p_n\}$, which means $x_i$ belongs to the $p_i$-th cluster ($1 \le p_i \le K$). If the loss function of the clustering is
\begin{equation}
   g(p)=\sum_{i=1}^n (x_i-c_{p_i})^2 ~,
\end{equation}
where $c_1 < c_2 < \cdots < c_K$ are cluster centers,  an optimal solution $p$ satisfies: $1=p_1 \le p_2 \le \cdots \le p_n=K$.
\end{theorem}

\begin{proof}{}
Suppose the ascending array $\{c_i\}$ are the cluster centers for the optimal solution.
Let $c'_0\!=\!-\infty$, $c'_1\!=\!(c_1\!+\!c_2)/2$,  $c'_2\!=\!(c_2\!+\!c_3)/2$, $\cdots$, $c'_{K-1}\!=\!(c_{K-1}\!+\!c_K)/2$, $c'_K\!=\!\infty$. We can construct a clustering by setting $p_i\!=\!j$, if 
$c'_{j-1}\!\le\!x_i\!<\!c'_j$. 
Obviously, ${p_i}$ minimizes $g(p)$ and corresponds
to an optimal solution satisfying $1=p_1 \le p_2 \le \cdots \le p_n=K$. 
\end{proof}

Theorem \ref{thm:theorem1} infers that by suitable interval partition we can get the optimal clustering of the weights (see Fig. 3). Let $x_1 \le \cdots \le x_N$ be the sorted weights, and $G_{n,k}$ be the minimum loss for clustering the first $n$ weights into $k$ clusters. We have
\begin{figure}[b]
 \setlength{\abovecaptionskip}{1pt}
  \centering
    \includegraphics[width=3.3in]{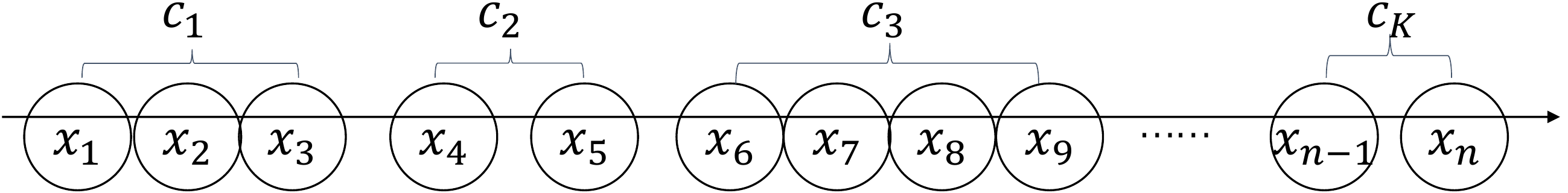}
  \caption{Illustration of the optimal solution of clustering $n$ scalars.}
    \label{fig:thm}
\end{figure}
\begin{equation}
    \label{eq:f}
    G_{n,k}= \min\limits_{k-1 \le i < n} G_{i,k-1}+h(i+1,n) ,
\end{equation}
where $k-1 \le i < n$ and $h(l, q)=\min_c \sum_{i=l}^q (x_i-c)^2$, it means the minimum clustering error (loss) for clustering $x_l, x_{l+1}, \cdots, x_t$ to one cluster. The situation  when $n=k$ is trivial, and we have $G_{k,k}=0$.  So, we consider the situations with $n>k$. To pursue the optimal clustering, we need to enumerate all possible $i$ which represents the largest index of scalar not belonging to the $k$-th cluster. Then, the minimum quantization error has two parts. One is the minimum quantization error that clustering $x_1,\cdots,x_{i}$ into $k-1$ clusters, i.e. $G_{i,k-1}$. The other part is the minimum quantization error that quantizing $x_{i+1},\cdots,x_n$ into one cluster, which is $h(i+1,n)$.  The latter part can be easily calculated, and the mean of scalars should be the cluster center. So, 
\begin{align}
    \label{eq:h}
    h(l,q)&=\sum_{i=l}^q(x_i-\frac{1}{q-l+1}\sum_{j=l}^q x_j)^2\\ \nonumber
    &=\sum_{i=l}^q x_i^2-(\sum_{i=l}^q x_i)^2/(q-l+1)
\end{align}{}

Based on (5) and (6) and utilizing the dynamic programming skills \cite{dreyfus1977art}, we derive  Algorithm \ref{alg:sq} for optimally solving the scalar quantization problem. It should be pointed out that, this algorithm can be easily extended to other kinds of clustering, such as that  the $L_2$ norm in the loss function  (4) is replaced with $L_1$ norm. The time complexity of Alg. 1 is  $O(n^2K)$. We will make sure that $n$ is not very large when applying it to the weight quantization.
\begin{algorithm}[h]
\caption{DP based scalar quantization}\label{alg:sq}
\begin{flushleft}
\textbf{Input:}  $n$ scalars $x_1 \le x_2  \cdots \le x_n$,  number of clusters $K$. \\
\textbf{Output:} The $K$ cluster centers in the optimal solution.
\end{flushleft}
\begin{algorithmic}[1]
\State Define an $n \times K$ array $G$, whose elements are all $\infty$.
\State Define an $n \times K$ array $L$, whose element $L_{i,k}$ represents the index of the smallest scalar within the same class as $x_i$, in the optimal solution of clustering the first $i$ scalars into $k$ classes.
\State $G_{0,0}\gets 0$
\For{$i \gets 1$ to $n$}
\State $s_1 \gets 0, ~ s_2 \gets 0$ \Comment{temporary variable}
\For{$j \gets i$ downto $1$} 
\State $s_1 \gets s_1+x_j, ~ s_2 \gets s_2+x_j^2$

 \Comment{$s_1=\sum_{k=j}^i x_k, ~ s_2=\sum_{k=j}^i x_k^2$ ~ }
\State $c \gets s_2-s_1^2/(i-j+1)$ \Comment{Eq. (\ref{eq:h})}
\For{$k \gets 1$ to $K$}
\If{$G_{i,k} > G_{j-1,k-1}+c$}
\State $G_{i,k} \gets G_{j-1,k-1}+c$ \Comment{Eq. (\ref{eq:f})}
\State $L_{i,k} \gets j$

\EndIf
\EndFor
\EndFor

\EndFor

\State centers $\gets [~ ]$
\While{$K > 0$}
\State $l \gets L_{n,K}$
\State centers.insert($\frac{\sum_{i=l}^n x_i}{n-l+1}$) \Comment{The center for last cluster.}
\State $n \gets L-1, ~ K \gets K-1$
\EndWhile
\State \Return centers
\end{algorithmic}
\end{algorithm}


\subsection{More Details of DP-Net}
For an FC layer with an $n_{fc} \times m_{fc}$ matrix, we divide the parameters into $n_{fc}$ parts, each part contains $n=m_{fc}$ weights, which means we cluster the weights row by row. For a convolutional layer with an $n_{conv} \times m_{conv} \times h_{conv} \times w_{conv}$ tensor, we divide the parameters into $n_{conv}$ parts and each part contains $n=m_{conv} \times h_{conv} \times w_{conv}$ weights. In this way, the number of weights on which we do clustering will be no more than $10000$ for most mainstream DNN. So, the computational time for the DP based algorithm is much less than that for training the DNN. We call this quantization manner a fine-grained scalar quantization. The compression ratio of the fine-grained quantization is less than that of quantizing all weights as a whole, but the drop is small.  We can see from (\ref{eq:eq1}) that the summation of  $n_i \log_2 K$ terms dominates the denominator. Thus, the drop of compression ratio is negligible.

In spite of small drop of compression ratio, the fine-grained quantization can improve the accuracy. Clustering an FC layer with an $n\!\times\!m$ matrix into $8$ classes with the fine-grained quantization can be seen as clustering $nm$ scalars into $8n$ clusters. This greatly expand the  parameter space, compared with clustering all the weights into $8$ classes as a whole.

\subsection{Extension for Compressing Robust DNN Model}
The idea of making DP-Net applicable to  the robust DNN models
is to add the quantization error, i.e. the minimum loss $G_{n,K}$, as a regularization term to the loss function optimized during the training process.
We consider the state-of-the-art TRADES (TRadeoff-inspired Adversarial DEfense via Surrogate-loss minimization) model for robust DNN \cite{zhang2019theoretically}. The  new formulation for training becomes:
\begin{align}
    &\min_{W, C} \{ L(f(X; W), Y) +\max_{X' \in B(X, \epsilon)}\gamma L(f(X; W), f(X'; W)) \nonumber \\
    &+ \lambda \sum_{i=1}^m\sum_{j=1}^{n_i} \min_{1 \le k \le K} (W_{i,j}-C_{i,k})^2 \},
\end{align}
where $f(X; W)$ is the output vector of learning model given parameters $W$, $L(x, y)$ is the cross-entropy loss function and $\gamma$ is a regularization parameter to trade off between accuracy and robustness. With this formulation, we can train a clustering-friendly robust model.

\subsection{Inference Acceleration}

There are existing work on accelerating neural network on FPGA \cite{han2016eie} and CPU \cite{sotoudeh2019c3}. In \cite{han2016eie}, the DNN model is compressed  by scalar quantization so that it becomes small enough to be stored in SRAM, thus increasing the speed of memory access. In addition to this benefit, specific technique for DP-Net can be developed for more inference acceleration.

In the inference stage the major computation can be decomposed as matrix-vector multiplications, for the FC layer or the convolutional layer. Each matrix-vector multiplication consists of many vector dot products. For simplicity, we just consider accelerating the vector dot product.
For two $n$-dimensional vectors $a, b$, normally we need to perform $n$ additions and $n$ multiplications to make the dot product. With DP-Net, one vector, say $b$, is presented as $[c_{p_1},c_{p_2},\cdots, c_{p_n}]^T$, where $c$ is a $K$-dimensional vector and $p$ is the $n$-dimensional index vector ($1 \le p_i \le K$). Then, we have $a^T b=\sum_{k=1}^K c_k ( \sum_{i \in S_k} a_i )$, where $S_i$ is the set of indices $j$ for which $p_j=i$, and we only need to perform $K$ multiplications. This brings acceleration. If DP-Net is deployed on FPGA, significant speedup can be expected, as multiplication  is executed much slower than addition in FPGA. Notice $n$ is usually larger than $1000$, while $K$ is  no more than $16$ in our experiments.

\section{Experimental Results}
In this section, we present the experimental results to demonstrate the effectiveness of DP-Net. We conducted experiments on CIFAR-10 \cite{krizhevsky2009learning} and compared DP-Net with two baselines for compressing TTConv \cite{garipov2016ultimate} and FreshNet \cite{chen2016compressing}. The first baseline is Deep K-means \cite{wu2018deep}, which is the work most relevant to ours. The second baseline is using the Lloyd's  algorithm to do clustering during the training of DP-Net. We call it KMeans-Net. Furthermore, we used the three-stage pipeline to compress Wide ResNet \cite{zagoruyko2016wide}. 
Then, we did experiments on the ImageNet dataset, for compressing GoogleNet. We also applied our method to the robust models. The results show that our extended DP-Net works well and has good performance. Finally, we show the result of the custom accelerator for DP-Net on FPGA. 
In all experiments, the compression ratio (CR) is obtained with (\ref{eq:eq1}) and then rounded to an integer. 


\subsection{Experiment Setup}
DP-Net has only two hyperparameters: $\lambda$, the regularizier factor in (\ref{eq:eq3}); $t$, the clustering frequency during training. We choosed $\lambda=100$ for all experiments if not explicitly stated. The value of $t$ varies for different datasets because it affects the number of training epochs for the model to converge. 

\subsubsection{On CIFAR-10}
We choosed $t=20$ for all experiments on CIFAR-10 with normal model and used SGD with cosine annealing for training. The learning rate reduces from 0.05 to 0.01 for TTConv and FreshNet. As for Wide ResNet, the learning rate is initialized at 0.1. The inference accuracy of the pretrained models we obtained for TTConv, FreshNet and Wide ResNet are $91.45\%$, $87.51\%$, $93.58\%$, respectively. They are higher or equal to those reported in \cite{wu2018deep}. For TTConv and Wide ResNet, we trained 300 epochs. Because FreshNet is prone to overfitting \cite{chen2016compressing}, we just trained 150 epochs and obtained a better result than training 300 epochs. The number of training epochs used for DP-Net is consistent with the pretrained model. We use random cropping, random horizontal flipping and cutout for data augmentation.

\subsubsection{On ImageNet} 

Our experiment settings are consistent with the PyTorch example\footnote{https://github.com/pytorch/examples/tree/master/imagenet} except training epochs. We use the pretrained GoogleNet model provided by PyTorch, whose top-1 accuracy is $69.78\%$ and top-5 accuracy is $89.53\%$. The PyTorch's official example trained 90 epochs, and made the learning rate decay 10X every 30 epochs. We just trained 30 epochs and made the learning rate decay 10X every 10 epochs because the pretrained model provided a good initial solution. We choose $t=3$ for ImageNet because $3$ epochs are enough for SGD to converge when the cluster centers are fixed.

\subsubsection{For Robust Model} 
We did experiments on the robust models trained by TRADES \cite{zhang2019theoretically}. The training settings and evaluating settings are consistent with their public code\footnote{https://github.com/yaodongyu/TRADES}. We choose $t=5$ because their training epochs are small. At first, we used a small network which consists of four convolutional layers and three FC layers. This network is proposed by \cite{carlini2017towards}. We called it SmallCNN and trained it on the MNIST dataset. The pretained model achieved 99.46$\%$ accuracy on testing dataset and 96.87$\%$ accuracy under a powerful attack algorithm named PGD \cite{kurakin2016adversarial}. Then, we trained a robust ResNet-18 model on CIFAR-10, which achieved 92.15$\%$ accuray on testing dataset and 40.29$\%$ accuracy under the PGD attack.

\subsubsection{Custom Accelerator}
As a hardware platform, the Zynq SoC (XC7Z045FFG900-2) is used, mounted on the ZC706-evaluation-board, and operating at 100MHz. It offers 218,600 Look-Up Tables (LUTs), 437,200 Flip-Flops, 545 units of 36k block RAM (BRAM) and 900 DSP48E units. The Vivado 2018.2 is used for RTL behavioral simulation.
\subsection{Result on CIFAR-10}
In this subsection, we show the effectiveness of the proposed DP-Net for compressing several DNNs on CIFAR-10.
\subsubsection{Compressing TT-Conv}

The TT-Conv model \cite{garipov2016ultimate} contains six convolutional layers and one FC layer. The authors of TT-Conv used Tensor Train Decomposition to compress the convolutional layer by 4X while the accuracy decreases by $2\%$. Deep K-means \cite{wu2018deep} achieves a better result while compressing 4X, with less accuracy loss. Our experimental results are listed in Table \ref{tab:TTConv}. From it we see that with only 3 bits used to represent a scalar (equivalent to CR of 10X), there is no or negeligible accuracy loss.

The results in Table \ref{tab:TTConv} also demonstrate two interesting phenomena. First, the accuracy of DP-Net with 3 bits is 1.18$\%$ higher than KMeans-Net, which means that using the Lloyd's algorithm instead of DP greatly reduces the accuracy. This phenomenon verifies the value of the proposed DP based algorithm. Second, our DP-Net performs better than the pretrained network even after the quantizing process, possibly because the regularized factor can prevent network from overfitting, like what a $L_2$-norm regularized factor often behaves. Therefore, we believe that performing a suitable weight quantization can not only reduce the size of model, but also prevent overfitting and improve performance.
\begin{table}[h]
 \setlength{\abovecaptionskip}{1pt}
 \setlength{\belowcaptionskip}{1pt}
\begin{center}
\caption{The results of compressing TT-Conv. $\Delta$ means the change of accuracy compared to the pretrained model.
}
\label{tab:TTConv}
\small{
\begin{tabular}{|c|c|c|}
    \hline
     Model & CR & $\Delta (\%)$\\
    \hline
    TT Decomposition \cite{garipov2016ultimate} & 4 & -2.0\\
    \hline
    Deep K-means \cite{wu2018deep} & 2 & +0.05\\
    Deep K-means \cite{wu2018deep} & 4 & -0.04\\
    \hline
    KMeans-Net (3 bits) & 10 & -0.87\\
    DP-Net (3 bits) & 10 & \textbf{+0.31}\\
    \hline
\end{tabular}
}
\end{center}
\end{table}

The $\lambda$ and $t$ are hyperparameters in our method. We conducted two experiments to study the sensitivity of our method to them, with the results plotted in Fig. \ref{fig:exp}. From the figure we see that DP-Net is not very sensitive to the hyperparameters. The phenomenons mentioned previously still exist. DP-Net always produces higher accuracy than KMeans-Net under the same configurations in all experiments. And, most DP-Net  models with 3-bit representation performs better than their pretrained models. This shows that a clustering-friendly network with appropriate number of clusters may have better generalization than normal network, which make our method meaningful even when compression is not needed.

\begin{figure}[h]
 \setlength{\abovecaptionskip}{1pt}
  \centering
    \subfigure[]
    {
    \includegraphics[height=1.2in]{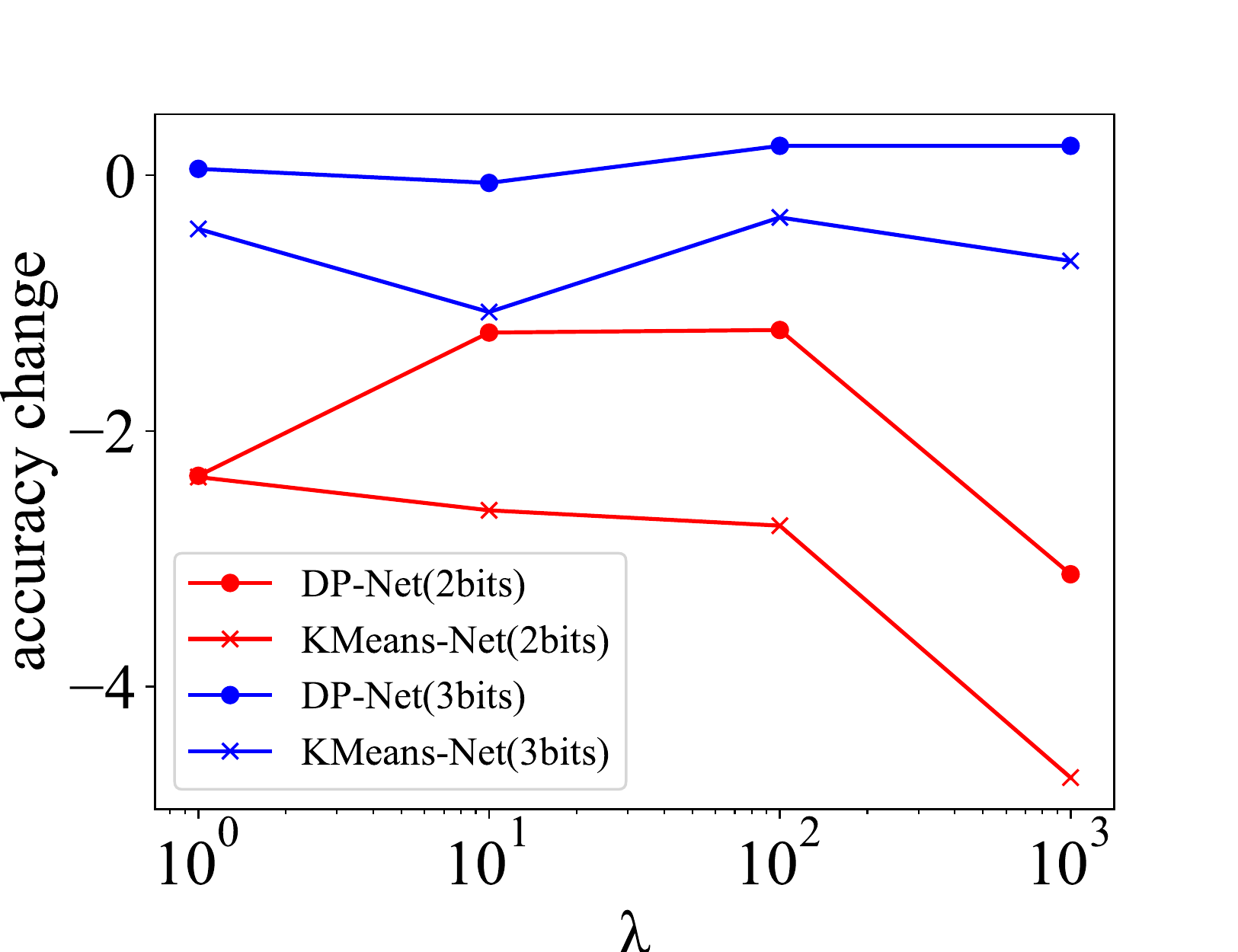}
    \label{fig:fixf}
    }  
    \subfigure[]
    {
    \includegraphics[height=1.2in]{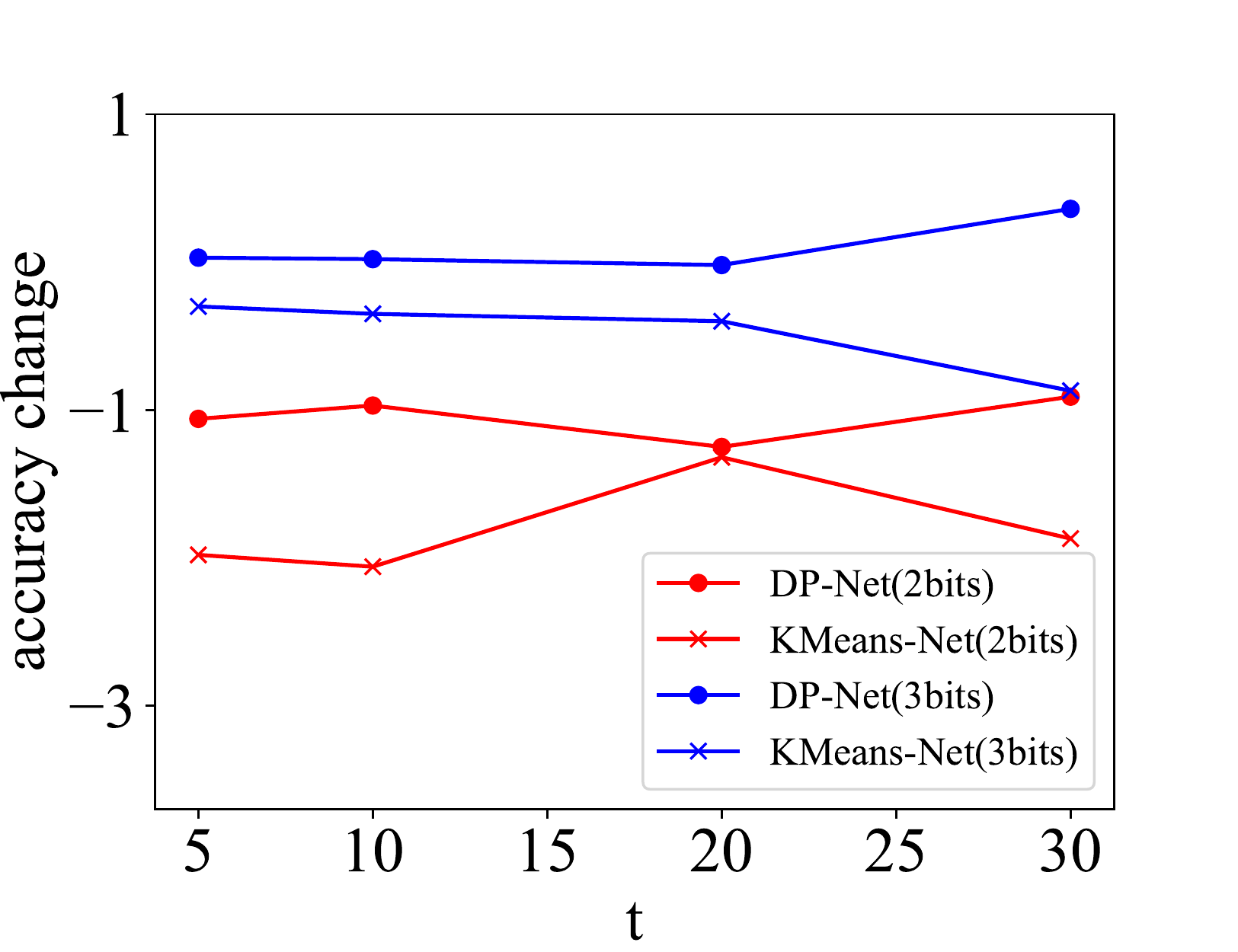}
    \label{fig:fixc}
    }
   \caption{The results of compressing TT-Conv. (a) The accuracy change for varied $\lambda$ with $t=20$. (b) The accuracy change for varied $t$ with $\lambda=100$.}
  \label{fig:exp}
\end{figure}

\subsubsection{Compressing FreshNet}
In \cite{chen2016compressing}, it is observed that the learned convolutional weights are smooth and low-frequency. So, they quantized the weights of their self-designed model called FreshNet on frequency domain, and then trained the quantized network. This model's accuracy decreased by $6.51\%$ when the CR is 16; with the same CR the accuracy is only reduced by $1.3\%$ if Deep K-means is used. This verifies that training a clustering-friendly network is better than training a quantized network. 
The comparison between DP-Net and related work are shown in Table \ref{tab:FreshNet}. From it we see that  DP-Net's accuracy is the highest, and much better than the methods in \cite{chen2016compressing} and the Deep K-means \cite{wu2018deep} with same  CR equal to 16. We think this shows that the scalar quantization is better in preserving accuracy, which is consistent with the experimental results on TT-Conv. 
\begin{table}[h]
 \setlength{\abovecaptionskip}{1pt}
 \setlength{\belowcaptionskip}{1pt}
\begin{center}
\caption{The results of compressing FreshNet. $\Delta$ means the change of accuracy compared to the pretrained model.}
\label{tab:FreshNet}
\small{
\begin{tabular}{|c|c|c|}
    \hline
     Model & CR & $\Delta (\%)$\\
    \hline
    Hashed Net \cite{chen2015compressing} & 16 & -9.79\\
    FreshNet \cite{chen2016compressing} & 16 & -6.51\\
    \hline
    Deep K-means \cite{wu2018deep} & 16 & -1.3\\
    \hline
    DP-Net (2 bits) & 16 & \textbf{-0.57}\\
    KMeans-Net (2 bits) & 16 & -0.76 \\
    \hline
\end{tabular}
}
\end{center}
\end{table}

\subsubsection{Compressing Wide ResNet}
We compressed Wide ResNet \cite{zagoruyko2016wide} by combining pruning, quantization and Huffman coding (the three-stage pipeline). The sparsity for each layer and the pruning method  are the same as what was done in \cite{wu2018deep}, where a total CR=47 is achieved. So, the only difference is  we replace the Deep K-means with our DP-Net for the  weight quantization step. The results are listed in Table \ref{tab:Wide}. It shows that when the CR reaches 50, the  accuracy of the three-stage pipeline including Deep K-means suffered significant decrease. In contrast, using the proposed DP-Net yields an accuracy loss less than 3$\%$ even at a compression ratio of 77X. Our accuracy drop is also remarkably less than that using Deep K-means at CR=50. 
\begin{table}[h]
 \setlength{\abovecaptionskip}{1pt}
 \setlength{\belowcaptionskip}{1pt}
\begin{center}
\caption{The results of compressing Wide ResNet with the three-stage pipelines. $\Delta$ means the change of accuracy compared to the pretrained model.}
\label{tab:Wide}
\small{
\begin{tabular}{|c|c|c|}
    \hline
     Model & CR & $\Delta (\%)$\\
    \hline
%
    Deep K-means \cite{wu2018deep} & 47 & -2.23\\

    Deep K-means \cite{wu2018deep} & 50 & -4.49\\
    \hline
    DP-Net (3 bits) & 50 & -1.71\\
    
    DP-Net (2 bits) & \textbf{77} & -2.94\\
    \hline
\end{tabular}
}
\end{center}
\end{table}

\subsection{Result on ImageNet}

We also evaluated our method on GoogleNet trained on ImageNet ILSVRC2012 dataset. Both Deep K-means and the proposed DP-Net are used to compress GoogleNet. 
Table \ref{tab:google} shows DP-Net is also better than Deep K-means on large-scale dataset. The accuracy of Deep K-means decreased quickly when CR$>$4. So, the accuracy of Deep K-means was not reported in \cite{wu2018deep} for larger CR. In contrast, the accuracy of DP-Net is improved even if we use 4 bits to represent a scalar value, which implies a 7X compression ratio. Furthurmore, even if GoogleNet is compressed by 10X, the accuracy of DP-Net remains higher than that of Deep K-means with CR=4.
\begin{table}[h]
 \setlength{\abovecaptionskip}{1pt}
 \setlength{\belowcaptionskip}{1pt}
\begin{center}
\caption{The results of compressing GoogleNet. $\Delta^{\dag}$ and $\Delta^{\ddag}$ are the changes of top1-accuracy and top5-accuracy compared to the pretrained model, respectively.}
\label{tab:google}
\small{
\begin{tabular}{|c|c|c|c|}
    \hline
     Model & CR & $\Delta^{\dag} (\%)$ & $\Delta^{\ddag} (\%)$\\
    \hline
%
%
    Deep K-means \cite{wu2018deep} & 4 & -1.95& -1.14\\
    \hline
    DP-Net (3 bits) & 10 & -1.56& -0.88\\
    
    DP-Net (4 bits) & 7 & \textbf{+0.3}& \textbf{+0.2}\\
    \hline
\end{tabular}
}
\end{center}
\end{table}
\subsection{Result for Robust Model}
To test the effectiveness of extended DP-Net on the robust model, we compressed the SmallCNN for MNIST and ResNet-18 for CIFAR-10. The models are trained by TRADES \cite{zhang2019theoretically}, where a minimax loss function is used to make the model robust. Because there is few work on compressing a robust model, we consider a  baseline called DP-Net WR for comparison. DP-Net WR directly quantizes the weights by DP without training the clustering-friendly network.  Table \ref{tab:robust} shows the accuracy loss of SmallCNN on MNIST and ResNet-18 on CIFAR-10, respectively. From them, we see that compressing a robust model is more challenging. For example, directly quantizing the SmallCNN to 2 bits  makes 0.36$\%$ accuracy drop on natural images, but the accuracy drop on adversarial examples becomes 5.86$\%$. 
The results in Table \ref{tab:robust} show that the extended DP-Net can compress the robust models by 14X and 16X with negligible accuracy loss.
In the results of DP-Net WR, $|\Delta_{adv}|< |\Delta_{nat}|$ for  ResNet-18, because the original accuracy on adversarial examples is much less than the original accuracy on natural images ($40.29\%$ versus $92.15\%$). 
\begin{table}[h]
 \setlength{\abovecaptionskip}{1pt}
 \setlength{\belowcaptionskip}{1pt}
\begin{center}
\caption{The results of compressing the robust SmallCNN (on MNIST) and ResNet-18 (on CIFAR-10). $\Delta_{nat}$ and $\Delta_{adv}$ are the accuracy changes for natrual images and adversarial images (under PGD attack) compared to the pretrained model, respectively.}
\label{tab:robust}
\small{
\begin{tabular}{|@{~}c@{~}|@{~}c@{~}|@{~}c@{~}|@{~}c@{~}|@{~}c@{~}|}
    \hline
     Model & CR & Model & $\Delta_{nat} (\%)$ & $\Delta_{adv} (\%)$\\
    \hline
    DP-Net  WR (2 bits) & 14 & SmallCNN & -0.36& -5.86\\
    \hline
    Extended DP-Net (2 bits) & \textbf{14} & SmallCNN & \textbf{-0.05}& \textbf{-2.05}\\
    \hline
    DP-Net WR (2 bits) & 16 & ResNet-18 & -22.19& -16.15\\
    \hline
    Extended DP-Net (2 bits) & \textbf{16} & ResNet-18 & \textbf{-1.45}& \textbf{+0.41}\\
    \hline
\end{tabular}
}
\end{center}
\end{table}
\subsection{Custom Accelerator}
To demonstrate the specific technique proposed in Section III.D, we test the runtime of DNN inference with the DP-Net deployed on FPGA.  The experiments involve an FC layer and a convolutional layer from GoogleNet compressed by DP-Net with 4-bit representation, and they are multiplied by random vectors. For the convolutional layer with $n_{conv} \times m_{conv} \times h_{conv} \times w_{conv}$ weights, we first reshape it to an $n_{conv} \times m_{conv} h_{conv} w_{conv}$ matrix. 
The baseline approach for matrix-vector multiplication takes in the matrix represented by 32-bit floating-point numbers stored in BRAM and/or distributed RAM, and the random vector stored in distributed RAM (ROM) to facilitate fast access. It adopts a pipelined design technique, and performs floating-point multiply-add operation by rows. The  operation is implemented with the Floating-point (7.1) IP Core provided by Xilinx, based on DSP Slice Full Usage.
For our approach based on DP-Net, the quantized/compressed matrix is small enough to be stored in the distributed RAM synthesized by LUTs. And, the technique in Section III.D is implemented to perform the matrix-vector multiplication.
The experimental results are listed in Table \ref{tab:speed}.
It shows that the custom accelerator for our DP-Net is at least 5X faster than the the baseline approach without compression.
In this experiment, 4-bits representation means that 16 clusters are used for preserving the accuracy of 
GoogleNet. 
For other DNN models, the weights can be clustered into fewer classes with DP-Net, which in turn would lead to larger compression ratio and more inference acceleration on FPGA.   

\begin{table}[h]
 \setlength{\abovecaptionskip}{1pt}
 \setlength{\belowcaptionskip}{1pt}
\begin{center}
\caption{The time of a matrix-vector multiplication on FPGA.}
\label{tab:speed}
\small{
\begin{tabular}{|@{~}c@{~}|@{~}c@{~}|c|@{~}c@{~}|@{~}c@{~}|}
    \hline
    Matrix Source & Matrix Size & Algorithm & Time (ms) & Speedup \\
    \hline
   \multirow{2}*{\Centerstack{FC layer}} & \multirow{2}*{\Centerstack{$1000\times 1024$}} & Baseline & 61.5 & -- \\
   \cline{3-5}
      &   & Ours & 11.9 & \textbf{5.1X} \\
    \hline
    \multirow{2}*{\Centerstack{Convolutional \\ layer}} & \multirow{2}*{\Centerstack{$384 \times 1728$}} & Baseline & 39.8 & -- \\ \cline{3-5}
      &   & Ours & 7.60 & \textbf{5.2X} \\
    \hline 
\end{tabular}
}
\end{center}
\end{table}
\section{Conlusion}

In this paper, a DNN compression scheme called DP-Net is proposed, which includes training a clustering-friendly network and using a dynamic programming (DP) based algorithm to obtain the optimal solution of weight quantization. Exhaustive  experiments have been carried out to show that the DP based algorithm is  better than the existing method for weight quantization. And, DP-Net can be combined with other compression methods. An experiment  on Wide ResNet demonstrates this and shows 77X compresion ratio with less than 3\% accuracy drop. DP-Net is also extended to compress the robust DNN models, showing 16X compression with negligible accuracy loss. Lastly, a technique based on DP-Net for inference acceleration is proposed. The experiment on FPGA shows that it brings 5X speedup to the computation associated with  FC layers and convolutional layers.

\bibliographystyle{IEEEtran}
\bibliography{IEEEabrv,IEEEexample}
\end{document}